\documentclass[a4,a4wide]{article}
\usepackage{aaai}
\usepackage{textcomp}
\usepackage{times}
\usepackage{helvet}
\usepackage{courier}
\usepackage{xcolor}
\usepackage{amssymb}
\usepackage{amsmath}
\usepackage{amsthm}
\usepackage{xcolor}

\newcommand{\ie}{i.e.\ }

\newcommand{\wrt}{w.r.t.\ }
\newcommand{\T}{\textbf{T}}

\newcommand{\q}{\mbox{$\mid\hspace{-2.4pt}\sim$}}
\newcommand{\alc}{\mathcal{ALC}}
\newcommand{\alcio}{\mathcal{ALCIO}}
\newcommand{\alcqo}{\mathcal{ALCQO}}

\newcommand{\alcf}{\mathcal{ALCF}}
\newcommand{\alco}{\mathcal{ALCO}}
\newcommand{\logicplus}{\textbf{$\alc$+T$^+$}}
\newcommand{\logicplusminio}{\textbf{$\alcio$+T$^+_{\mathsf{min}}$}}
\newcommand{\logicplusminqo}{\textbf{$\alcqo$+T$^+_{\mathsf{min}}$}}
\newcommand{\logicplusmin}{\textbf{$\alc$+T$^+_{\mathsf{min}}$}}
\newcommand{\logicOmin}{\textbf{$\alco$+T$_{\mathsf{min}}$}}
\newcommand{\logic}{\textbf{$\alc$+T}}
\newcommand{\logicmin}{\textbf{$\alc$+T$_{\mathsf{min}}$}}
\newcommand{\logicminio}{\textbf{$\alcio$+T$_{\mathsf{min}}$}}
\newcommand{\logicminqo}{\textbf{$\alcqo$+T$_{\mathsf{min}}$}}
\newcommand{\tbox}{\mathcal{T}}
\newcommand{\abox}{\mathcal{A}}
\newcommand{\kbase}[3]{#1=(#2,#3)}
\newcommand{\I}{\mathcal{I}}
\newcommand{\J}{\mathcal{J}}

\newcommand{\Min}{\mathsf{Min}}
\newcommand{\prefs}{<_{\mathcal{L}_\T}}
\newcommand{\pref}{<_{\mathcal{L}_\T}^+}
\newcommand{\minset}[2]{#1^{\Box^-}_{\mathcal{L}_{\T_{#2}}}}
\newcommand{\nexpnp}{\textbf{$\mathsf{NExp^{NP}}$}}
\newcommand{\expt}{\textbf{$\mathsf{EXPTIME}$}}
\newcommand{\circs}[3]{\mathsf{Circ_{#1}}(#2,#3)}
\newcommand{\cpattern}[4]{\mathsf{#1}=(#2,#3,#4)}
\newcommand{\sig}[1]{\mathsf{sig}(#1)}
\newcommand{\cset}{\mathsf{N_C}}
\newcommand{\rset}{\mathsf{N_R}}
\newcommand{\iset}{\mathsf{N_I}}
\newcommand{\tset}{\mathsf{N_T}}
\newcommand{\minsat}[3]{#1 \models_{\mathsf{min}}^{#3} #2}

\newcommand{\model}[1]{\mathcal{M}_#1}
\newcommand{\alcsigreduct}[1]{#1|_{\alc}}

\theoremstyle{plain}
\newtheorem{theorem}{Theorem}
\newtheorem{lemma}[theorem]{Lemma}
\newtheorem{corollary}[theorem]{Corollary}
\theoremstyle{definition}
\newtheorem{definition}[theorem]{Definition}
\newtheorem{example}[theorem]{Example}
\theoremstyle{remark}
\newtheorem*{remark}{Remark}

\begin{document}
\nocopyright
\title{On the Non-Monotonic Description Logic $\logicmin$}
\author{Oliver Fern\'andez Gil\thanks{Supported by DFG Graduiertenkolleg 1763 (QuantLA).}\\
University of Leipzig\\Department of Computer Science\\
fernandez@informatik.uni-leipzig.de\\
}
\maketitle
\begin{abstract}
\begin{quote}
 In the last 20 years many proposals have been made to incorporate non-monotonic reasoning into description logics, ranging from approaches based on 
 default logic and circumscription to those based on preferential semantics. In particular, the 
non-monotonic description logic $\logicmin$ uses a combination of 
 the preferential semantics with minimization of a certain kind of concepts, which represent 
atypical instances of a class of elements. One of its drawbacks is that it suffers from the problem 
known as the \emph{property blocking inheritance}, which
  can be seen as a weakness from an inferential point of view. In this paper we propose an extension of $\logicmin$, namely $\logicplusmin$, with the purpose to solve the mentioned problem. 
  In addition, we show the close connection that exists between $\logicplusmin$ and 
concept-circumscribed knowledge bases. Finally, we study the complexity of deciding the classical 
reasoning tasks in $\logicplusmin$.
\end{quote}
\end{abstract}

\section{Introduction.}

 Description Logics (DLs) \cite{DBLP:conf/dlog/2003handbook} are a well-investigated family of logic-based knowledge 
representation formalisms. They can be used to represent knowledge of a problem domain in a 
structured and formal way. To describe this kind of knowledge each DL provides constructors that 
allow to build concept descriptions. A knowledge base consists of a TBox that states general 
assertions about the problem domain and an ABox that asserts properties about explicit individuals.

 Nevertheless, classical description logics do not provide any means to reason about exceptions. In 
the past 20 years research has been directed with the purpose to incorporate non-monotonic 
reasoning formalisms into DLs. In \cite{DBLP:journals/jar/BaaderH95}, an 
integration of Reiter's default logic \cite{DBLP:journals/ai/Reiter80} within the terminological 
language $\alcf$ is proposed and later extended in \cite{DBLP:journals/jar/BaaderH95a} to allow the 
use of priorities between default rules. Taking a different 
approach, \cite{DBLP:journals/jair/BonattiLW09} introduces circumscribed DLs and analyses in 
detail the complexity of reasoning in circumscribed extensions of expressive description logics. In addition, recent 
works \cite{DBLP:conf/jelia/CasiniS10,DBLP:conf/ausai/BritzMV11,DBLP:conf/dlog/GiordanoGOP13} 
attempt to introduce defeasible reasoning by extending DLs with preferential and rational semantics based on the 
KLM approach to propositional non-monotonic reasoning \cite{DBLP:journals/ai/LehmannM92}.

 In particular, the logic $\logicmin$ introduced in \cite{DBLP:journals/ai/GiordanoGOP13} combines the use of a preferential semantics and the minimization of a certain 
 kind of concepts. This logic is built on top of the description logic $\alc$ 
\cite{DBLP:journals/ai/Schmidt-SchaussS91} and is based on a typicality operator \textbf{T} whose 
intended meaning 
is to single out the \emph{typical} instances of a class $C$ of elements. The 
 underlying semantics of \textbf{T} is based on a preference relation over the domain. More 
precisley, classical $\alc$ interpretations are 
 equipped with a partial order over the domain elements setting a preference relation among them. 
Based on such an order, for instance, the set of \emph{typical birds} or 
 \textbf{T}$(\mathsf{Bird})$, comprises those individuals from the domain that are birds and minimal in the class of all birds with respect to the preference order. Using this operator, 
 the subsumption statement \textbf{T}$(\mathsf{Bird})\sqsubseteq \mathsf{Fly}$ expresses that 
\emph{typical birds fly}. In addition, the use of a minimal 
 model semantics considers models that minimize the atypical instances of $\mathsf{Bird}$. Then,  
when no information is given about whether a bird is able to fly or not, it is possible to 
\emph{assume} that it flies in view of the assertion \textbf{T}$(\mathsf{Bird})\sqsubseteq 
\mathsf{Fly}$.
 
 As already pointed out by the authors, the preferential order over the domain limits the logic $\logicmin$ in the sense that if a class $P$ is an exceptional case 
 of a superclass $B$, then no default properties from $B$ can be inherited by $P$ during the 
reasoning process, including those that are unrelated with the exceptionality 
 of $P$ with respect to $B$. For example:\begin{align}
           \mathsf{Penguin} &\sqsubseteq \mathsf{Bird} \notag \\
           \T(\mathsf{Bird}) &\sqsubseteq \mathsf{Fly} \sqcap \mathsf{Winged} \notag \\
           \T(\mathsf{Penguin}) &\sqsubseteq \neg \mathsf{Fly} \notag
       \end{align} It is not possible to infer that typical penguins have wings, even when the only reason for them to be exceptional with respect to birds is that they
        normally do not fly.
 
 In the present paper we extend the non-monotonic logic $\logicmin$ from \cite{DBLP:journals/ai/GiordanoGOP13} with the introduction of several preference relations. We
  show how this extension can handle the inheritance of defeasible properties, resembling the use of abnormality predicates from 
  circumscription \cite{DBLP:journals/ai/McCarthy86}. In addition, we show the close relationship 
between the extended non-monotonic logic $\logicplusmin$ and 
  \emph{concept-circumscribed} knowledge bases \cite{DBLP:journals/jair/BonattiLW09}. Based on such 
a relation, we provide a complexity analysis of the different reasoning tasks showing $\nexpnp$- 
completeness for 
  concept satisfiability and co-$\nexpnp$-completeness for subsumption and instance checking.
  
   Missing proofs can be found in the long version of the paper at 
http://www.informatik.uni-leipzig.de/\texttildelow fernandez/NMR14long.pdf.
 \section{The logic $\logicmin$.}
 We recall the logic $\logic$ proposed in \cite{DBLP:journals/ai/GiordanoGOP13} and its 
non-monotonic extension $\logicmin$. Let $\cset, \rset$ and $\iset$ be three countable sets of 
\emph{concept names, role names and 
individual names}, respectively. The language defined by $\logic$ distinguishes between normal 
concept descriptions and \emph{extended concept} descriptions which are formed according to the 
following syntax rules:
  \begin{align}
      C &::= A \: \mid \: \neg C \: \mid \: C \sqcap D \mid \: \exists r.C , 
\notag \\
      C_e &::= C \: \mid \: \T(A) \: \mid \: \neg C_e \: \mid \: C_e \sqcap D_e \notag     
  \end{align}where $A \in \cset, r \in \rset$, $C$ and $D$ are classical $\alc$ concept 
descriptions, $C_e$ 
and $D_e$ are extended concept descriptions, and \textbf{T} is the newly introduced operator. We 
use the 
usual abbreviations $C \sqcup D$ for $\neg(\neg C \sqcap \neg D)$, $\forall r.C$ for $\neg \exists 
r. \neg C$, $\top$ for $A \sqcup \neg A$ and $\bot$ for $\neg \top$.

 A knowledge base is a pair $\kbase{\mathcal{K}}{\tbox}{\abox}$. The TBox $\tbox$ contains 
subsumption statements $C \sqsubseteq D$ where $C$ is a classical $\alc$ concept or an extended 
concept of the form $\T(A)$, and $D$ is a 
classical $\alc$ concept. The Abox $\abox$ contains assertions of the form $C_e(a)$ and $r(a,b)$ 
where $C_e$ is an extended concept, $r\in \rset$ and $a,b \in \iset$. The assumption that the operator $\T$ is applied to concept names is 
without loss of generality. For a complex $\alc$ concept $C$, one can always introduce a fresh 
concept name $A_C$ which can be made equivalent to $C$ by adding the subsumption statements $A_C 
\sqsubseteq C$ and $C \sqsubseteq A_C$ to the background TBox. Then, $\T(C)$ can be equivalently 
expressed as $\T(A_C)$.

 In order to provide a semantics for the operator $\T$, usual $\alc$ interpretations are equipped 
with a preference relation $<$ over the domain elements:

 \begin{definition}[Interpretation in $\logic$]\label{logic_int_def}  
 An $\logic$ interpretation $\I$ is a tuple 
$(\Delta^\I,.^\I,<)$ where:
\begin{itemize}
 \item $\Delta^\I$ is the domain,
 \item $.^\I$ is an interpretation function that maps concept names to subsets of $\Delta^\I$ and role names to binary relations over $\Delta^\I$,
 \item $<$ is an irreflexive and transitive relation over 
$\Delta^\I$ that satisfies the following condition (\textbf{Smoothness Condition}): for all $S 
\subseteq \Delta^\I$ and for all $x \in S$, either $x \in \Min_{<}(S)$ or $\exists y \in 
\Min_{<}(S)$ such that $y < x$, with $\Min_{<}(S)=\{x \in S\mid\not\exists y\in S \text{ s.t. } y < x\}$.
  \end{itemize}
 \end{definition}   
  The operator $\T$ is interpreted as follows: $[\T(A)]^\I=\Min_{<}(A^\I)$. 
For arbitrary concept descriptions, $.^\I$ is inductively extended in the same way as for $\alc$ 
taking into account 
the introduced semantics for $\T$.

 As mentioned in \cite{DBLP:journals/ai/GiordanoGOP13,DBLP:journals/fuin/GiordanoOGP09}, $\logic$ 
is still monotonic and has several limitations. In the following we present the logic 
$\logicmin$, proposed in \cite{DBLP:journals/ai/GiordanoGOP13} as a non-monotonic extension of 
$\logic$, where a preference relation is defined between $\logic$ interpretations and only minimal 
models are considered.

 First, we introduce the modality $\Box$ as in \cite{DBLP:journals/ai/GiordanoGOP13}.

 \begin{definition}
    Let $\I$ be an $\logic$ interpretation and $C$ a concept description. Then, $\Box C$ is interpreted under $\I$ in the following way: \[(\Box C)^\I=\{x \in \Delta^\I 
\mid \text{ for all }y\in\Delta^\I \text{ if }y < x\text{ then }y\in C^\I\}\]
    We remark that $\Box C$ does not extend the syntax of $\logic$. The purpose of using it is to characterize elements of the domain with respect to whether all 
    their predecessors in $<$ are instances of $C$ or not. For example, $\Box \neg \mathsf{Bird}$ defines a concept such that $d \in (\Box \neg \mathsf{Bird})^\I$ if all 
    the predecessors of $d$, with respect to $<$ under the interpretation $\I$, are not instances of $\mathsf{Bird}$. Hence, it is not difficult to see that:
    \[[\T(\mathsf{Bird})]^\I=(\mathsf{Bird} \sqcap \Box\neg \mathsf{Bird})^\I\]
    Then, the idea is to prefer models that minimize the instances of $\neg\Box\neg \mathsf{Bird}$ in order to minimize the number of \emph{atypical} birds. 
 \end{definition}

 Now, let $\mathcal{L}_\T$ be a finite set of concept names occurring in the knowledge base. These are 
the concepts whose atypical instances are meant to be minimized. For each 
interpretation $\I$, the set $\minset{\I}{}$ represents all the instance of concepts of the form 
$\neg \Box \neg A$ for all $A \in \mathcal{L}_\T$. Formally,\[\minset{\I}{} = \{(x,\neg\Box\neg 
A)\mid x \in 
(\neg\Box\neg A)^\I,\text{ with } x \in \Delta^\I, A\in \mathcal{L}_\T\}\]

 Based on this, the notion of minimal models is defined in the following way.

\begin{definition}[Minimal models]\label{def_min_models}
   Let $\kbase{\mathcal{K}}{\tbox}{\abox}$ be a knowledge base and 
$\I=(\Delta^\I,.^\I,<_\I)$, $\J=(\Delta^\J,.^\J,<_\J)$
 be two interpretations. We say that $\I$ is preferred to $\J$ with 
respect to the set $\mathcal{L}_\T$ (denoted as $\I \prefs \J$), iff:
\begin{itemize}
 \item $\Delta^\I=\Delta^\J$,
  \item $a^\I=a^\J$ for all $a \in \iset$,
  \item $\minset{\I}{} \subset \minset{\J}{}$.
  \end{itemize}  
 An interpretation $\I$ is a minimal model of $\mathcal{K}$ with respect to $\mathcal{L}_\T$
(denoted as $\minsat{\I}{\mathcal{K}}{\mathcal{L}_\T}$) iff $\I \models \mathcal{K}$ and 
there is no model $\J$ of $\mathcal{K}$ such that $\J \prefs \I$.
\end{definition} 

 Based on the notion of minimal models, the standard reasoning tasks are defined for $\logicmin$.
 
   \begin{itemize}   
     \item \emph{Knowledge base consistency (or satisfiability):} A knowledge base $\mathcal{K}$ 
is consistent w.r.t. $\mathcal{L}_\T$, if there exists an interpretation $\I$ such that 
$\minsat{\I}{\mathcal{K}}{\mathcal{L}_\T}$.     
     \item \emph{Concept satisfiability:} An extended concept $C_e$ is satisfiable with respect 
to $\mathcal{K}$ if there exists a minimal model $\I$ of $\mathcal{K}$ w.r.t. $\mathcal{L}_\T$ such 
that 
$C_e^\I \neq \emptyset$.
      \item \emph{Subsumption:} Let $C_e$ and $D_e$ be two extended concepts. $C_e$ is subsumed by 
$D_e$ w.r.t. $\mathcal{K}$ and $\mathcal{L}_\T$, denoted as $\minsat{\mathcal{K}}{C_e 
\sqsubseteq D_e}{\mathcal{L}_\T}$, if 
$C_e^\I \subseteq D_e^\I$ for all minimal models $\I$ of $\mathcal{K}$.     
      \item \emph{Instance checking:} An individual name $a$ is an instance of an extended 
concept $C_e$ w.r.t. $\mathcal{K}$, denoted as $\minsat{\mathcal{K}}{C_e(a)}{\mathcal{L}_\T}$, 
if $a^\I \in C_e^{\I}$ in all the minimal models $\I$ of $\mathcal{K}$.
    \end{itemize}  

 Regarding the computational complexity, the case of \emph{knowledge base consistency} is not interesting in itself since the logic $\logic$ enjoys the finite model 
 property \cite{DBLP:journals/ai/GiordanoGOP13}. Note that if there exists a finite model $\I$ of $\mathcal{K}$, then the sets that are being minimized 
 are finite. Therefore, every descending chain starting from $\I$ with respect to $\prefs$ must be 
finite and a minimal model of $\mathcal{K}$ always exists. 
 Thus, the decision problem only requires to decide knowledge base consistency of the underlying monotonic logic $\logic$ which has been shown to be $\expt$-complete \cite{DBLP:journals/fuin/GiordanoOGP09}. 
 For the other reasoning tasks, a $\nexpnp$ upper bound is provided for $\emph{concept satisfiability}$ and a co-$\nexpnp$ upper bound for subsumption and instance checking \cite{DBLP:journals/ai/GiordanoGOP13}.

\section{Extending $\logicmin$ with more typicality operators.}
 
  As already mentioned in \cite{DBLP:journals/ai/GiordanoGOP13,DBLP:journals/fuin/GiordanoOGP09}, the use of a global relation to represent that one individual is more typical 
  than another one, limits the expressive power of the logic. It is not possible to express that an individual $x$ is more typical than an individual $y$ with respect 
  to some aspect $As_1$ and at the same time $y$ is more typical than $x$ (or not comparable to 
$x$) with respect to a different aspect $As_2$. This, for example, implies that a subclass cannot 
inherit 
  any property from a superclass, if the subclass is already exceptional with respect to one property of the superclass. This effect is also known as 
  \emph{property inheritance blocking} \cite{DBLP:conf/tark/Pearl90,DBLP:journals/ai/GeffnerP92}, and 
is a known problem in preferential extensions of DLs based on the KLM approach.
  
   We revisit the example from the introduction to illustrate this problem.
  \begin{example}\label{example_inheritance}
      Consider the following knowledge base:
       \begin{align}
           \mathsf{Penguin} &\sqsubseteq \mathsf{Bird} \notag \\
           \T(\mathsf{Bird}) &\sqsubseteq \mathsf{Fly} \sqcap \mathsf{Winged} \notag \\
           \T(\mathsf{Penguin}) &\sqsubseteq \neg \mathsf{Fly} \notag
       \end{align} Here, \emph{penguins} represent an exceptional subclass of \emph{birds} in the sense that they \emph{usually are unable to fly}. However, it might be intuitive to conclude that 
   they \emph{normally have wings} ($\T(\mathsf{Penguin}) \sqsubseteq \mathsf{Winged}$) since 
although birds fly because they have wings, having wings does not imply the ability to fly. 
   In fact, as said before, it is not possible to sanction this kind of conclusion in $\logicmin$. The problem is that due to the global character of the order $<$ among individuals of the domain, 
   once an element $d$ is assumed to be a \emph{typical} penguin, then automatically a more 
preferred individual $e$ must exist that is a \emph{typical} bird. This rules out the possibility 
to apply
    the non-monotonic assumption represented by the second assertion to $d$.
    \end{example}
    
     In relation with circumscription, this situation can be modelled using abnormality predicates 
to represent exceptionality with respect to different aspects 
\cite{DBLP:journals/ai/McCarthy80,DBLP:journals/ai/McCarthy86}. The following example shows a 
knowledge base which is defined using abnormality concepts similar as the examples 
in \cite{DBLP:journals/jair/BonattiLW09}.
\begin{example}\label{circ_example}
   \begin{align}
          \mathsf{Penguin} &\sqsubseteq \mathsf{Bird} \notag \\
           \mathsf{Bird} &\sqsubseteq \mathsf{Fly} \sqcup Ab_1 \notag \\
           \mathsf{Bird} &\sqsubseteq  \mathsf{Winged} \sqcup Ab_2 \notag \\
           \mathsf{Penguin} &\sqsubseteq \neg \mathsf{Fly} \sqcup Ab_{\mathsf{penguin}} \notag
      \end{align} The semantics of circumscription allows to consider only models that minimize the 
instances of the abnormality concepts. In this example, concepts $Ab_1$ and $Ab_2$ are used to 
represent birds that are atypical with respect to two independent aspects (i.e.: $\mathsf{Fly}$ and 
$\mathsf{Winged}$). If the minimization forces an individual $d$ to be a \emph{not abnormal} 
penguin (i.e.: $d$ is not an instance of $Ab_{\mathsf{penguin}}$), then it must be an instance 
of $Ab_1$, but at the same time nothing forces it to be an instance of $Ab_2$. Therefore, it is 
possible to assume that $d$ has wings because of the minimization of $Ab_2$.
\end{example}     

  In this paper, we follow a suggestion given in \cite{DBLP:journals/ai/GiordanoGOP13} that asks 
for the extension of the logic $\logicmin$ with 
  more preferential relations in order to express typicality of a class with respect to different aspects. We define the logic $\logicplus$ and its extension $\logicplusmin$ 
  in a similar way as for $\logic$ and $\logicmin$, but taking into account the possibility to use more than one typicality operator.
  
  We start by fixing a finite number of typicality operators $\T_1,\ldots,\T_k$. Classical concept 
descriptions and extended concept descriptions are defined by the following 
  syntax:\begin{align}
      C &::= A \: \mid \: \neg C \: \mid \: C \sqcap D \mid \: \exists r.C , 
\notag \\
      C_e &::= C \: \mid \: \T_i(A) \: \mid \: \neg C_e \: \mid \: C_e \sqcap D_e, \notag     
  \end{align} where all the symbols have the same meaning as in $\logic$ and $\T_i$ ranges over the 
set of typicality operators. The semantics is defined as an extension 
  of the semantics for $\logic$ that takes into account the use of more than one $\T$ operator.

\begin{definition}[Interpretations in $\logicplus$]\label{logicplus_int}
   An interpretation $\I$ in $\logicplus$ is a tuple 
$(\Delta^\I,.^\I,<_1,\ldots,<_k)$ where:
\begin{itemize}
 \item $\Delta^\I$ is the domain,
 \item $<_i$ $(1 \leq i \leq k)$ is an irreflexive and transitive relation over 
$\Delta^\I$ satisfying the Smoothness Condition.
  \end{itemize}
  
 Typicality operators are interpreted in the expected way with respect to the different preference 
relations over the domain: $[\T_i(A)]^\I=\Min_{<_i}(A^\I)$.
\end{definition}
 
 Similar as for $\logic$, we introduce for each preference relation $<_i$ an indexed box 
modality $\Box_i$ such that: \[(\Box_i C)^\I=\{x \in \Delta^\I \mid \forall y\in\Delta^\I: \text{ 
if }y <_i x\text{ then }y\in C^\I\}\] Then, the set of typical instances of a concept $A$ with 
respect to the $i^{th}$ typical operator can be expressed in terms of the indexed $\Box$ 
modalities:\[[\T_i(A)]^\I=\{x\in\Delta^\I\mid x\in (A \sqcap 
\Box_i\neg 
A)^\I\}\]

 Now, we define the extension of $\logicplus$ that results in the non-monotonic logic 
$\logicplusmin$. Let $\mathcal{L}_{\T_1},\ldots,\mathcal{L}_{\T_k}$ be $k$ finite 
 sets of concept names. Given an $\logicplus$ interpretation $\I$, the sets $\minset{\I}{i}$ are 
defined as:\[\minset{\I}{i}=\{(x,\neg\Box_i\neg A)\mid x\in(\neg\Box_i\neg A)^\I \land 
A\in\mathcal{L}_{\T_i}\}\] Based on these sets, we define the preference relation $\pref$ on 
$\logicplus$ interpretations that characterizes the non-monotonic semantics of $\logicplusmin$.   

 \begin{definition}[Preference relation]\label{prefrel_def}
    Let $\kbase{\mathcal{K}}{\tbox}{\abox}$ be a knowledge base and 
$\I=(\Delta^\I,.^\I,<_{i_1},\ldots,<_{i_k})$, $\J=(\Delta^\J,.^\J,<_{j_1},\ldots,<_{j_k})$
 be two interpretations. We say that $\I$ is preferred to $\J$ (denoted as $\pref$) with respect to 
the sets $\mathcal{L}_{\T_i}$, iff:
\begin{itemize}
 \item $\Delta^\I=\Delta^\J$,
  \item $a^\I=a^\J$ for all $a \in \iset$,
  \item $\minset{\I}{i} \subseteq \minset{\J}{i}$ for all $1 \leq i \leq k$,
  \item $\exists \ell$ s.t. $\minset{\I}{\ell}\subset\minset{\J}{\ell}$.
\end{itemize}

 An $\logicplus$ interpretation $\I$ is a minimal model of $\mathcal{K}$ (denoted as 
$\minsat{\I}{\mathcal{K}}{\mathcal{L}_{\T^+}}$) iff 
$\I \models \mathcal{K}$ and there exists no interpretation 
$\J$ such that: $\J\models\mathcal{K}$ and $\J\pref\I$. The different reasoning tasks are defined in the usual way, 
but with respect to the new entailment relation $\models_{\mathsf{min}}^{\mathcal{L}_{\T^+}}$.
 \end{definition}
 
  We revise Example \ref{example_inheritance} to show how to distinguish 
between a bird being typical with respect \emph{to being able to fly} or 
  \emph{to having wings}, in $\logicplusmin$. The example shows the use of two typicality operators $\T_1$ and $\T_2$, where $<_1$ and $<_2$ are the underlying preference
   relations.
  \begin{example}
     \begin{align}
          \mathsf{Penguin} &\sqsubseteq \mathsf{Bird} \notag \\
           \T_1(\mathsf{Bird}) &\sqsubseteq \mathsf{Fly} \notag \\
           \T_2(\mathsf{Bird}) &\sqsubseteq  \mathsf{Winged} \notag \\
           \T_1(\mathsf{Penguin}) &\sqsubseteq \neg \mathsf{Fly} \notag
      \end{align} 
      \end{example}
      
  In the example, we use two preference relations to express typicality of birds with respect to two different aspects independently. 
  The use of a second preference relation 
  permits that typical penguins can also be typical birds with respect to $<_2$. Therefore, it is possible to infer that typical penguins do have wings. Looking from 
  the side of individual elements: having the assertion $\mathsf{Penguin}(e)$, the minimal model semantics allows to assume that $e$ is a typical penguin and also a 
  typical bird with respect to $<_2$, even when a bird $d$ must exist such that $d$ is preferred to $e$ with respect to $<_1$.
  
   It is interesting to observe that the defeasible property \emph{not being able to fly}, for 
penguins, is stated with respect to $\T_1$. If instead, we use $\T_2(\mathsf{Penguin}) \sqsubseteq 
\neg \mathsf{Fly}$, there will be minimal models where $e$ is an instance of 
$\T_1(\mathsf{Bird})$ and others where it is an instance of $\T_2(\mathsf{Penguin})$. This implies 
that it will not be possible to infer for $e$, the defeasible properties corresponding to 
the most specific concept it belongs to. 

The same problem is realized, with respect to 
circumscription in Example \ref{circ_example}, where some minimal models prefer $e$ to be a 
\emph{normal} bird ($e \in \neg Ab_1$), while others consider $e$ as a \emph{normal} penguin 
($e \in \neg Ab_\mathsf{penguin}$). To address this problematic about specificity, one 
needs to use priorities between the minimized concepts (or abnormality 
predicates) \cite{DBLP:journals/ai/McCarthy86,DBLP:journals/jair/BonattiLW09}. 

In contrast, for the formulation in the example, the semantics 
induced by the preferential order $<_1$ does not allow to have interpretations where $e \in 
\mathsf{Penguin}$, $e \in \T_1(\mathsf{Bird})$ and $e \not\in \T_1(\mathsf{Penguin})$, i.e., the 
treatment of specificity comes for free in the semantics of the logic.

\section{Complexity of reasoning in $\logicplusmin$.}
 
  In the following, we show that reasoning in $\logicplusmin$ is $\nexpnp$-complete for \emph{concept satisfiability} and $\mathsf{co}$-$\nexpnp$-complete for \emph{subsumption} 
  and \emph{instance checking}. As a main tool we use the close correspondence that exists 
  between \emph{concept-circumscribed} knowledge bases in the DL $\alc$ 
\cite{DBLP:journals/jair/BonattiLW09} and $\logicplusmin$ knowledge bases. In fact, this relation 
has been pointed out in \cite{DBLP:journals/ai/GiordanoGOP13} with respect 
  to the logic $\logicmin$. However, on the one hand, the provided mapping from $\logicmin$ into 
\emph{concept-circumscribed} knowledge bases is not polynomial, and instead a tableaux calculus is 
used to 
show the upper bounds for the main reasoning tasks in $\logicmin$. On the other hand, the relation 
in 
the opposite direction is only given with respect to the logic $\logicOmin$, which extends 
$\logicmin$ by allowing the use of nominals. 
  
  First, we improve the mapping proposed in \cite{DBLP:journals/ai/GiordanoGOP13} by giving a 
simpler polynomial reduction, that translates $\logicplusmin$ knowledge bases into 
\emph{concept-circumscribed} knowledge bases while preserving the entailment relation under the 
translation. Second, we show that using more than one typicality operator, it is possible to reduce 
the problem of concept satisfiability for \emph{concept-circumscribed} knowledge bases in $\alc$, 
into the concept satisfiability problem for $\logicplusmin$.
  
  We start by introducing circumscribed 
knowledge bases in the DL $\alc$, as defined in \cite{DBLP:journals/jair/BonattiLW09}. We obviate 
the use of priorities between minimized predicates.
  
  \begin{definition}
     A circumscribed knowledge base is an expression of the form $\circs{CP}{\tbox}{\abox}$ where $\cpattern{CP}{M}{F}{V}$ is a circumscription pattern such that 
     $M, F, V$ partition the predicates (i.e.: concept and role names) used in $\tbox$ and $\abox$. The set $M$ identifies those concept names whose extension is minimized, 
     $F$ those whose extension must remain fixed and $V$ those that are free to vary. A circumscribed knowledge base where $M\cup F\subseteq \cset$ is called a \emph{concept-circumscribed} 
     knowledge base.
     
     To formalize a semantics for circumscribed knowledge bases, a preference relation $<_{\mathsf{CP}}$ is defined on interpretations by setting $\I <_\mathsf{CP} \J$ iff:
     \begin{itemize}
      \item $\Delta^\I = \Delta^\J$,
      \item $a^\I=a^\J$ for all $a \in \iset$,
      \item $A^\I=A^\J$ for all $A \in F$,
      \item $A^\I \subseteq A^\J$ for all $A \in M$ and there exists an $A' \in M$ such that 
${A'}^\I \subset {A'}^\J$.
     \end{itemize} An interpretation $\I$ is a model of $\circs{CP}{\tbox}{\abox}$ if $\I$ is a model of $(\tbox,\abox)$ and there is no model $\I'$ of $(\tbox,\abox)$
     with $\I' <_{\mathsf{CP}} \I$. The different reasoning tasks can be defined in the same way as above.
  \end{definition}
 
 Similar as for circumscribed knowledge bases in \cite{DBLP:journals/jair/BonattiLW09}, one can show that concept satisfiability, subsumption and instance checking can be polynomially reduced to 
 one another in $\logicplusmin$. However, to reduce instance checking into concept satisfiability slightly different technical details have to be considered.
 
 \begin{lemma}\label{lemm_reasoning_red}
    Let $\kbase{\mathcal{K}}{\tbox}{\abox}$ be an $\logicplus$ knowledge base, $C_e$ an extended 
concept, $\mathcal{L}_{\T_1}, \ldots, \mathcal{L}_{\T_k}$ be finite sets of concept names and $A$ a 
fresh concept name not occurring in $\mathcal{K}$ and $C_e$. Then, 
    $\minsat{\mathcal{K}}{C_e(a)}{\mathcal{L}_{\T^+}}$ iff $\neg \T_{k+1}(A) \sqcap \neg C_e$ is 
unsatisfiable w.r.t. 
    $\kbase{\mathcal{K}'}{\tbox \cup \{\top\sqsubseteq A\}}{\abox \cup \{(\neg \T_{k+1}(A))(a)\}}$, 
where $\mathcal{L}_{\T_{k+1}}=\{A\}$.
 \end{lemma}
 
  Note that this reduction requires the introduction of an additional typicality operator 
$\T_{k+1}$. 
Nevertheless, this does not represent a problem in terms of complexity
   since, as it will be shown in the following, the complexity does not depend on the 
number of typicality operators $k$ whenever $k \geq 2$.
 \subsection{Upper Bound.} 
 Before going into the details of the reduction we need to define the notion of a signature.
 \begin{definition}  
    Let $\tset$ be the set of all the concepts of the form $\T_i(A)$ where $A \in \cset$. A 
signature $\Sigma$ for $\logicplus$ is a finite subset of $\cset \cup \rset 
\cup \tset$. We denote by $\alcsigreduct{\Sigma}$ the set $\Sigma \setminus \tset$.
   
    The signature $\sig{C_e}$ of an extended concept $C_e$ is the set of all concept names, role 
names and concepts from $\tset$ that occur in $C_e$. Similarly, the signature $\sig{\mathcal{K}}$ 
of an $\logicplus$ knowledge base $\mathcal{K}$ is the union of 
the signatures of all concept descriptions occurring in $\mathcal{K}$. Finally, we denote by 
$\sig{E_1,\ldots,E_m}$ the set $\sig{E_1} \cup \ldots\cup \sig{E_m}$, where each $E_i$ is either an 
extended concept or a knowledge base.
 \end{definition}

  Let $\kbase{\mathcal{K}}{\tbox}{\abox}$ be an $\logicplus$ knowledge base, 
$\mathcal{L}_{\T_1},\ldots,\mathcal{L}_{\T_k}$ 
 finite sets of concept names and $\Sigma$ be any signature with $\sig{\mathcal{K}} \subseteq 
\Sigma$. A corresponding circumscribed knowledge base 
$\circs{CP}{\tbox'}{\abox'}$, with $\mathcal{K'}=(\tbox',\abox')$, is built in the following way:

 \begin{itemize}  
   \item For every concept $A$ such that it belongs to some set $\mathcal{L}_{\T_i}$ or $\T_i(A) 
\in \Sigma$, a fresh concept name $A_i^*$ is introduced. These concepts are meant to 
represent the atypical elements with respect to $A$ and $<_i$ in $\mathcal{K}$, i.e., $\neg \Box_i 
\neg A$.
   \item Every concept description $C$ defined over $\Sigma$ is transformed into a 
concept $\bar{C}$ by replacing every occurrence of $\T_i(A)$ by $(A \sqcap \neg A_i^*)$.
  \item The TBox $\tbox'$ is built as follows:
     \begin{itemize}
        \item $\bar{C} \sqsubseteq \bar{D} \in \tbox'$ for all $C \sqsubseteq D \in \tbox$,
        \item For each new concept $A_i^*$ the following assertions are included in 
$\tbox'$:\begin{align}
             A_i^* &\equiv \exists r_i.(A \sqcap \neg A_i^*) \\                              
              \exists r_i.A_i^*&\sqsubseteq A_i^* 
         \end{align}             
       where $r_i$ is a fresh role symbol, not occurring in $\Sigma$, introduced to represent the relation $<_i$.
      \end{itemize}      
  \item $\abox'$ results from replacing every assertion of the form $C(a)$ in $\tbox$ by 
the assertion $\bar{C}(a)$.
  \item Let $\mathcal{L}_\T$ be the 
set: \[\bigcup_{j=1}^k 
\bigcup_{A \in \mathcal{L}_{\T_j}} A_j^*\] then, the concept circumscription pattern $\mathsf{CP}$ 
is defined as $\cpattern{CP}{M}{F}{V}=(\mathcal{L}_\T,\emptyset, 
\alcsigreduct{\Sigma} \cup \{A_i^* \mid A_i^* \not\in\mathcal{L}_\T\} \cup \{r_i\mid 1\leq i\leq 
k\})$.
 \end{itemize}
 
  One can easily see that the provided encoding is polynomial in the size of $\mathcal{K}$. The use of the signature $\Sigma$ is just a technical detail and since it is 
  chosen arbitrarily, one can also select it properly for the encoding of the different reasoning tasks.
 
  The idea of the translation is to simulate each order $<_i$ with a relation $r_i$ and at the same 
time fulfill the semantics underlying the $\T_i$ operators. 
  The first assertion, $A_i^* \equiv \exists r_i.(A \sqcap \neg A_i^*)$, intends to express that the 
atypical elements with respect to $A$ and $<_i$ are those, 
  and only those, that have an $r_i$-successor $e$ that is an instance of $A$ and at the same time a 
not atypical $A$, \ie, $e \in \T_i(A)$. Indeed, this is a consequence from 
  the logic $\logicplusmin$ because the order $<_i$ is transitive. However, since it is not possible to enforce transitivity of $r_i$ when translated into $\alc$, we need 
  to use the second assertion $\exists r_i.A_i^* \sqsubseteq A_i^*$. This prevents to have the following situation:\[d \in A_1^*\quad d \in B \sqcap \neg B_1^*\quad(d,e) \in r_1\quad 
  e \in A \sqcap \neg A_1^*\quad e \in B_1^*\] In the absence of assertion $(2)$, this would be 
consistent with respect to $\tbox'$, but it would not satisfy the aim of 
  the translation since the typical $B$-element $d$ would have a predecessor ($r_i$-successor) $e$ 
which is atypical with respect to $B$. In fact, the translation provided 
  in \cite{DBLP:journals/ai/GiordanoGOP13} also deals with this situation, but all the possible cases are asserted explicitly yielding an exponential encoding.
 
  The following auxiliary lemma shows that a model of $(\tbox',\abox')$ can always be transformed into a model, that only differs in the interpretation of $r_i$, and  
  $(r_i)^{-1}$ is irreflexive, transitive and \emph{well-founded}.  
 \begin{lemma}\label{lemma_transform_r}
    Let $\I$ be an $\alc$ interpretation such that $\I \models 
(\tbox',\abox')$. Then, there exists $\J$ such that $\J \models 
(\tbox',\abox')$, $X^{\I}=X^{\J}$ for all $X \in \alcsigreduct{\Sigma} \cup \bigcup A_i^*$, and 
for each $r_i$ we have: ${\left({r_i}^{\J}\right)}^{-1}$ is irreflexive, transitive and 
well-founded.
 \end{lemma}

  Since \emph{well-foundedness} implies the Smoothness Condition, the previous lemma allows us 
to assume (without loss of generality) that ${\left({r_i}^{\I}\right)}^{-1}$ is irreflexive, 
transitive and satisfies the Smoothness Condition for every model $\I$ of $\mathcal{K}'$.

 Now, we denote by $\model{\mathcal{K}}$ the set of models of $\mathcal{K}$ and by 
$\model{\mathcal{K'}}$ the set of models 
 of $\mathcal{K}'$. With the help of the previous lemma, we show that there exists a \emph{one-to-one} correspondence between 
 $\model{\mathcal{K}}$ and $\model{\mathcal{K'}}$. We start by defining a mapping $\varphi$ that 
transforms $\logicplus$ interpretations 
 into $\alc$ interpretations.

\begin{definition}\label{def_mapping}
    We define a mapping $\varphi$ from $\logicplus$ interpretations into $\alc$ interpretations 
such that $\varphi(\I)=\J$ iff:    
    \begin{itemize}
       \item $\Delta^\J = \Delta^\I$,
       \item $X^\J=X^\I$ for each $X \in \alcsigreduct{\Sigma}$,
       \item $(A_i^*)^\J=(\neg\Box_i\neg A)^\I$ for each fresh concept name $A_i^*$,
       \item $(r_i)^\J=(<_i)^{-1}$ for all $i, 1\leq i \leq k$,
       \item $a^\J=a^\I$, for all $a \in \iset$.
    \end{itemize}
\end{definition}

 \begin{remark}
    We stress that interpretations are considered only with respect to concept and role names occurring in $\Sigma$ for $\logicplus$, and 
 $\alcsigreduct{\Sigma} \cup \{A_i^*\} \cup \{r_i\}$ for $\alc$. All the other concept and role names from $\cset$ and $\rset$ are not relevant to distinguish one interpretation 
 from another one. This is, if $\I$ and $\J$ are two $\logicplus$ interpretations, then 
$\I\equiv\J$ iff $X^\I=X^\J$ for all $X \in \Sigma \cap (\cset \cup \rset)$ and 
 $(<_i)^\I=(<_i)^\J$ for all $i, 1\leq i \leq k$. The same applies for $\alc$ interpretations, but with respect to $\alcsigreduct{\Sigma} \cup \{A_i^*\} \cup \{r_i\}$. 
 \end{remark}

 Next, we show that $\varphi$ is indeed a bijection from $\model{\mathcal{K}}$ to $\model{\mathcal{K'}}$.
  
\begin{lemma}\label{lemma_relation}
   The mapping $\varphi$ is a bijection from $\model{\mathcal{K}}$ to $\model{\mathcal{K'}}$, 
such that for every $\I \in \model{\mathcal{K}}$ and each extended concepts $C_e$ defined over 
$\Sigma$: $C_e^\I=(\bar{C_e})^{\varphi(\I)}$.
\end{lemma}

 \begin{proof}
    First, we show that for each $\I \in \model{\mathcal{K}}$ it holds that: $\varphi(\I) \in 
\model{\mathcal{K'}}$. Let $\mathcal{I}=(\Delta^{\mathcal{I}},.^\mathcal{I},<_1,\ldots,<_k)$ be 
a model of $\mathcal{K}$ and assume that $\varphi(\I)=\J$. We observe that since 
$[\T_i(A)]^\I=(A\sqcap \Box_i \neg A)^\I$, then by definition of 
$\varphi$ it follows that:
\begin{equation}\label{eq_typic}
  [\T_i(A)]^\I=(A \sqcap \neg 
A_i^*)^\J  
\end{equation} Consequently, one can also see that for every 
extended concept $C_e$ defined over $\Sigma$ and every element $d \in \Delta^\I$:
\begin{equation}\label{eq_extend}
d \in C_e^\I \text{ iff } d \in (\bar{C_e})^\J\                                                     
\end{equation} This can be shown by a straightforward induction on 
the structure of $C_e$ where the base cases are $A$ and $\T_i(A)$. Hence, it follows that 
$C_e^\I=(\bar{C_e})^\J$ for every extended concept $C_e$ defined over $\Sigma$.

 Now, we show that $\J \models (\tbox',\abox')$. From (\ref{eq_extend}), it is clear that $\J 
\models \bar{C} \sqsubseteq \bar{D}$ for all $\bar{C} 
\sqsubseteq \bar{D} \in \tbox'$. In addition, since $a^\J=a^\I$ for all $a\in\iset$, $\J$ satisfies 
each assertion in $\abox'$. It is left to show that each GCI in 
$\tbox'$ containing an 
occurrence of a fresh role $r_i$ is also satisfied by $\J$. For each $d \in \Delta^\I$ and 
concept name $A_i^*$, it holds:
\begin{align}
   d\in (A_i^*)^\J &\text{iff } d \in (\neg \Box_i \neg A)^\I \notag \\
                   &\text{iff } \exists e \in \Delta^\I \text{ s.t. } e <_i d \text{ and } e \in 
[\T_i(A)]^\I \notag \\
                   &\text{ iff } (d,e) \in (r_i)^{\J} \text{ and } e \in (A \sqcap \neg A_i^*)^\J 
\quad\text{ {by 
(\ref{eq_typic})} } \notag \\
&\text{ iff } d \in (\exists r_i.(A \sqcap \neg A_i^*))^\J \notag
\end{align} The case for the second GCI ($\exists r_i.A_i^* \sqsubseteq A_i^*$) can be shown  in a 
very similar way. Thus, $\J \models (\tbox',\abox')$ and consequently $\varphi$ is a function from $\model{\mathcal{K}}$ into 
$\model{\mathcal{K'}}$.

 Second, we show that for any model $\J$ of $\mathcal{K'}$ (\ie $\J \in \model{\mathcal{K'}}$), 
there exists $\I \in \model{\mathcal{K}}$ with $\varphi(\I)=\J$. Let $\J$ be an arbitrary model of 
$\mathcal{K}'$, we build an $\logicplus$ interpretation 
$\I=(\Delta^{\I},.^{\I},<_1,\ldots,<_k)$ in the following way:
\begin{itemize}
 \item $\Delta^{\I}=\Delta^\J$,
 \item $X^{\I}=X^{\J}$ for each $X \in \alcsigreduct{\Sigma}$,
 \item $<_i={\left({r_i}^{\J}\right)}^{-1}$ for all $i, 1\leq i \leq k$,
 \item $a^{\I}=a^{\J}$, for all $a \in \iset$.
\end{itemize}
Next, we show that $(\neg \Box_i \neg A)^{\I}=(A_i^*)^{\J}$. Assume that 
$d \in (\neg \Box_i \neg A)^{\I}$ for some $d \in \Delta^{\I}$, then there exists $e 
<_i d$ such that $e \in A^{\I}$ and $e \in [\T_i(A)]^{\I}$. This means
  that for all $f <_i e$(or $(e,f) \in r_i^{\J}$): $f \not \in A^{\I}$. Hence, $e \in A^{\J}$ and 
$e \not \in (A_i^*)^{\J}$. All in all, we have $(d,e) \in r_i^\J$ and $e \in (A \sqcap \neg 
A_i^*)^\J$, therefore $d \in (A_i^*)^{\J}$.
  Conversely, assume that $d \in (A_i^*)^{\J}$. Assertion $(1)$ in $\tbox'$ implies that there 
exists $e$ such that $(d,e)\in (r_i)^{\J}$ and $e \in A^{\J}$. By construction of $\I$ we have $e 
<_i d$ and $e \in A^{\I}$. Thus, $d \in (\neg \Box_i \neg 
A)^{\I}$ and we can conclude that $(\neg \Box_i \neg A)^{\I}=(A_i^*)^{\J}$. Having this, it 
follows that $\varphi(\I)=\J$. In addition, similar as for equation 
(\ref{eq_typic}), we have:\begin{equation}\label{eq_typic1}
  [\T_i(A)]^{\I}=(A \sqcap \neg 
A_i^*)^{\J}  
\end{equation} A similar reasoning, as above yields that $\I \models 
\mathcal{K}$. This implies that $\varphi$ is surjective. It is not difficult to see, from the definition of $\varphi$, that it is also injective. 
Thus, $\varphi$ is a bijection from $\model{\mathcal{K}}$ to 
$\model{\mathcal{K'}}$.
 \end{proof}
 
 The previous lemma establishes a \emph{one to one} correspondence between $\model{\mathcal{K}}$ 
and $\model{\mathcal{K'}}$. Then, since $\mathcal{K}$ is an arbitrary $\logicplus$ knowledge base, 
Lemma \ref{lemma_relation} also implies that knowledge base consistency in $\logicplus$ can be 
polynomially reduced to knowledge base consistency in $\alc$, which is $\expt$-complete 
\cite{DBLP:conf/dlog/2003handbook}.
\begin{theorem}
   In $\logicplus$, deciding knowledge base consistency is $\expt$-complete.
\end{theorem}

 In addition, since $\alc$ enjoys the finite model property, this is also the case for $\logicplus$. 
Using the same argument given before for $\logic$ and $\logicmin$, deciding knowledge base 
consistency in $\logicplusmin$ reduces to the same problem with respect to the underlying monotonic 
logic $\logicplus$. Therefore, we obtain the following theorem.

\begin{theorem}\label{theo_consistency}
    In $\logicplusmin$, deciding knowledge base consistency is $\expt$-complete.
 \end{theorem}
 
 Now, we show that $\varphi$ is not only a bijection from $\model{\mathcal{K}}$ to 
$\model{\mathcal{K'}}$, but it is also \emph{order-preserving} with respect to $\pref$ and 
$<_\mathsf{CP}$.
 
 \begin{lemma}\label{lemma_isomorphism}
    Let $\I$ and $\J$ be two models of $\mathcal{K}$. Then, $\I \pref 
\J$ iff $\varphi(\I) <_{\mathsf{CP}} \varphi(\J)$.
   \begin{proof}
Assume that $\I \pref \J$. Then, for all $A \in 
\mathcal{L}_{T_i}$ we have that $(\neg \Box_i \neg A)^\I \subseteq (\neg \Box_i \neg A)^\J$ and in 
particular, for some $j$ and $A' \in \mathcal{L}_{T_j}$ we have $(\neg \Box_j \neg A')^\I \subset 
(\neg \Box_j \neg A')^\J$. By definition of $\varphi$, we know that $(\neg \Box_i \neg 
A)^\I=(A_i^*)^{\varphi(\I)}$. Hence, for all $A_i^* \in M$ we have that $(A_i^*)^{\varphi(\I)} 
\subseteq (A_i^*)^{\varphi(\J)}$ and $({A'}_j^*)^{\varphi(\I)} \subset ({A'}_j^*)^{\varphi(\J)}$. 
Thus, $\varphi(\I) <_{\mathsf{CP}} \varphi(\J)$. The other direction can be shown in the same way.
   \end{proof}
 \end{lemma}
 
 The following lemma is an easy consequence from the previous one and the fact that $\varphi$ is 
bijection (which implies 
 that $\varphi$ is invertible).

 \begin{lemma}\label{lemma_min_models}
    Let $\I$ and $\J$ be $\logicplus$ and $\alc$ interpretations, respectively. Then, 
    \begin{align}
   \minsat{\I}{\mathcal{K}}{\mathcal{L}_{\T^+}} \:\:&\text{ iff }\:\: \varphi(\I) \models 
\circs{CP}{\tbox'}{\abox'} \tag{a} \\
   \J \models 
\circs{CP}{\tbox'}{\abox'} \:\:&\text{ iff }\:\: \minsat{\varphi^{-1}(\J)}{\mathcal{K}}{\mathcal{L}_{\T^+}} \tag{b}
\end{align}
 \end{lemma}

 Thus, we have a correspondence between minimal models of $\mathcal{K}$ and 
models of $\circs{CP}{\tbox'}{\abox'}$. Based on this, it is easy to reduce each reasoning task 
from $\logicplusmin$ into the equivalent task with respect to \emph{concept-circumscribed} 
knowledge bases. 
The following lemma states the existence of such a reduction for concept satisfiability, the cases 
 for subsumption and instance checking can be proved in a very similar way.
 
 \begin{lemma}\label{lemma_upper_bound}
    An extended concept $C_0$ is satisfiable \wrt to $\mathcal{K}$ and 
$\mathcal{L}_{\T_1},\ldots,\mathcal{L}_{\T_k}$ iff $\bar{C_0}$ is satisfiable in 
$\circs{CP}{\tbox'}{\abox'}$.
 \end{lemma}
 \begin{proof}
     Let us define $\Sigma$ as $\sig{\mathcal{K},C_0}$.
     
    $(\Rightarrow)$ Assume that $\I$ is a minimal model of $\mathcal{K}$ with $C_0^\I \neq 
\emptyset$. The application of Lemma \ref{lemma_min_models} tells us that $\varphi(\I) 
\models \circs{CP}{\tbox'}{\abox'}$. In addition, from Lemma \ref{lemma_relation} we have that 
$C_0^\I=(\bar{C_0})^{\varphi(\I)}$. Thus, $\bar{C_0}$ is satisfiable in 
$\circs{CP}{\tbox'}{\abox'}$.
     $(\Leftarrow)$ The argument is similar, but using $\varphi^{-1}$.
     \end{proof}
    Finally, from the complexity results proved in \cite{DBLP:journals/jair/BonattiLW09} for the 
different reasoning tasks with respect to \emph{concept-circumscribed} knowledge 
    bases in $\alc$, we obtain the following upper bounds.
    
    \begin{theorem}
       In $\logicplusmin$, it is in $\nexpnp$ to decide concept satisfiability and in co-$\nexpnp$ to decide subsumption and instance checking.
    \end{theorem}
 \subsection{Lower Bound.}
 
 To show the lower bound, we reduce the problem of concept satisfiability with respect to 
\emph{concept-circumscribed} knowledge 
bases in $\alc$, into the concept satisfiability problem in $\logicplusmin$. It is enough to consider 
\emph{concept-circumscribed} knowledge bases of the form $\circs{CP}{\tbox}{\abox}$ with 
$\cpattern{CP}{M}{F}{V}$ where $\abox=\emptyset$ and $F=\emptyset$. The problem of deciding concept satisfiability for this class of circumscribed 
knowledge bases has been shown to be \nexpnp-hard for $\alc$ \cite{DBLP:journals/jair/BonattiLW09}. 
In order to do that, we modify the reduction provided 
in \cite{DBLP:journals/ai/GiordanoGOP13} which shows \nexpnp-hardness for concept satisfiability in 
$\logicOmin$. 

 Before going into the details, we assume without loss of generality that each minimized concept 
occurs in the knowledge base:
\begin{remark}
   Let $\circs{CP}{\tbox}{\abox}$ be a circumscribed knowledge base. If $A \in M$ and $A$ does not occur in $(\tbox,\abox)$, 
   then for each model $\I$ of $\circs{CP}{\tbox}{\abox}$: $A^\I=\emptyset$.
\end{remark}
 Given a circumscribed knowledge base $\mathcal{K}=\circs{CP}{\tbox}{\abox}$ (where 
$\mathsf{CP}$ is of the previous form) 
 and a concept description $C_0$, we define a corresponding $\logicplus$ knowledge base 
$\kbase{\mathcal{K}'}{\tbox'}{\abox'}$ using \textbf{two} typicality operators in the following way.

Let $M$ be the set $\{M_1,\ldots,M_q\}$. Similarly as in \cite{DBLP:journals/ai/GiordanoGOP13}, individual names $c$ and 
$c_{m_i}$ (one for each $M_i \in M$) and a fresh concept name $D$ are introduced.  Each $\alc$ concept description $C$ is transformed into $C^*$ 
inductively by introducing $D$ into concept descriptions of the form $\exists r.C_1$, i.e.: 
$(\exists r.C_1)^*=\exists r.(D \sqcap 
C_1^*)$ (see \cite{DBLP:journals/ai/GiordanoGOP13} for precise details). 

 Similar as in \cite{DBLP:journals/ai/GiordanoGOP13}, we start by adding the following GCIs to the 
TBox $\tbox'$:\begin{align}
             D \sqcap C_1^* &\sqsubseteq C_2^* \text{ if } C_1 \sqsubseteq C_2 \in \tbox \\                              
             D \sqcap M_i &\sqsubseteq \neg \T_1(M_i) \text{ for all } M_i \in M 
         \end{align} The purpose of using these subsumption statements is to establish a 
correspondence between the minimized concept names $M_i$, from the circumscription side, with the 
underlying concepts 
 $\neg \Box_1 \neg M_i$ on the $\logicplusmin$ side, such that the minimization of the $M_i$ 
concepts can be simulated by the minimization of $\neg \Box_1 \neg M_i$. The 
 individual names $c_{m_i}$ are introduced to guarantee the existence of typical $M_i$'s in view of 
assertion $(7)$. The concept $D$ plays the role to distinguish the 
 elements of the domain that are not mapped to those individual names by an interpretation. 
 
  Note that if under an interpretation $\I$ an element $d$ is an instance of $D$ and $M_i$ at the 
same time, then it has to be an instance of $\neg\T_1(M_i)$ and therefore an instance of 
$\neg\Box_1\neg M_i$ as well. Hence, it is important that whenever $d$ becomes an instance of 
$\Box_1 \neg M_i$ in a preferred interpretation to $\I$, it happens because $d$ becomes an instance 
of $\neg M_i$ while it is still an instance of $D$. In order to force this effect during the 
minimization, the interpretation of the concept $D$ should remain fixed in some way. As pointed out 
in \cite{DBLP:journals/ai/GiordanoGOP13}, this
  seems not to be possible in $\logicmin$ and that is why the reduction is realized for $\logicOmin$ 
where nominals are used with that purpose.
  
  In contrast, for $\logicplusmin$ this effect on $D$ can be simulated by introducing a second typicality operator $\T_2$, setting $\mathcal{L}_{\T_1}=M, \mathcal{L}_{\T_2}=\{A\}$ and adding
   the following two assertions to $\tbox'$:\begin{align}
             \top &\sqsubseteq A \\                              
             \neg D &\sqsubseteq \neg \T_2(A) 
         \end{align} where $A$ is a fresh concept name. Note that if an element $d$ becomes a 
$(\neg D)$-element, it automatically becomes a $(\neg \Box_2 \neg A)$-element.

 The ABox $\abox'$ contains the following assertions: 
  \begin{itemize}
   \item $D(c)$,
   \item for each $M_i \in M$:
       \begin{itemize}
        \item $(\neg D)(c_{m_i})$,
        \item $(\T_1(M_i))(c_{m_i})$,
        \item $(\neg M_j)(c_{m_i})$ for all $j \neq i$.
       \end{itemize}
  \end{itemize} 
  
  Finally, a concept description $C_0'$ is defined as $D \sqcap C_0^*$.
    
  \begin{lemma}\label{reduction_lemma}
     $C_0$ is satisfiable in $\circs{CP}{\tbox}{\abox}$ iff $C_0'$ is 
satisfiable \wrt $\kbase{\mathcal{K}'}{\tbox'}{\abox'}$ in $\logicplusmin$.
  \end{lemma}
  
   \begin{proof}
       Details of the proof are deferred to the long version of the paper.
   \end{proof}

   Since the size of $\mathcal{K}'$ is polynomial with respect to the size of $\mathcal{K}$, the application of the previous lemma yields the following result.  
  \begin{theorem}\label{theo_hard}
     In $\logicplusmin$, concept satisfiability is $\nexpnp$-hard.
  \end{theorem}
  
   Since concept satisfiability, subsumption and instance checking are polynomially interreducible (see Lemma \ref{lemm_reasoning_red}), Theorem \ref{theo_hard} yields co-$\nexpnp$ lower bounds for the subsumption and the instance 
   checking problem.
   
   \begin{corollary}\label{cor_main}
      In $\logicplusmin$, it is $\nexpnp$-complete to decide concept satisfiability and co-$\nexpnp$-complete to decide subsumption and instance checking.   
   \end{corollary}
   
  Finally, we remark that the translations provided between $\logicplusmin$ and 
\emph{concept-circumscribed} knowledge bases do not depend on the classical constructors 
  of the description logic $\alc$. Therefore, the same translations can be used for the 
more expressive description logics $\alcio$ and $\alcqo$. From the complexity results obtained 
  in \cite{DBLP:journals/jair/BonattiLW09} for circumscription in $\alcio$ and $\alcqo$ , we also 
obtain the following corollary. 
  
  \begin{corollary}\label{cor_main}
      In $\logicplusminio$ and $\logicplusminqo$, it is $\nexpnp$-complete to decide concept satisfiability and co-$\nexpnp$-complete to decide subsumption and instance checking.   
   \end{corollary}
   
   Moreover, from the lower bound obtained in \cite{DBLP:journals/ai/GiordanoGOP13} for $\logicOmin$, the results also apply for the logics $\logicminio$ and $\logicminqo$.
   \begin{corollary}
      In $\logicminio$ and $\logicminqo$, it is $\nexpnp$-complete to decide concept satisfiability and co-$\nexpnp$-complete to decide subsumption and instance checking.      
   \end{corollary}
   
\section{Conclusions}

 In this paper, we have provided an extension of the non-monotonic description logic $\logicmin$, 
by 
adding the possibility to use more than one preference relation over the domain 
 elements. This extension, called $\logicplusmin$, allows to express typicality of a class of 
elements with respect to different aspects in an ``independent'' way. Based on this, 
 a class of elements $P$ that is exceptional with respect to a superclass $B$ regarding a specific aspect, could still be not exceptional with respect to different unrelated aspects. 
 The latter permits that defeasible properties from $B$ not conflicting with the exceptionality of 
$P$, can be inherited by elements in $P$. As already observed in the paper, this is not possible in 
the logic $\logicmin$.

  In addition, we have introduced translations that show the close relationship between 
$\logicplusmin$ and \emph{concept-circumscribed} knowledge bases in $\alc$. First, the provided 
translation from $\logicplusmin$ into \emph{concept-circumscribed} knowledge bases is polynomial, 
in contrast with the exponential translation given in \cite{DBLP:journals/ai/GiordanoGOP13} for 
$\logicmin$. Second, the translation presented for the opposite direction shows how to encode 
circumscribed knowledge base, by using two typicality operators and no nominals.

Using these translations, we were able 
  to determine the complexity of deciding the different reasoning tasks in $\logicplusmin$. We have 
shown that it is $\nexpnp$-complete to decide concept satisfiability and 
  co-$\nexpnp$-complete to decide subsumption and instance checking. Moreover, the same 
translations can be used for the corresponding extensions of $\logicplusmin$ into more 
  expressive description logics like $\alcio$ and $\alcqo$. The results also apply for extensions of $\logicmin$ with respect to the underlying description logics, in view of the hardness result shown for $\logicOmin$ in \cite{DBLP:journals/ai/GiordanoGOP13}.
  
  As possible future work, the exact complexity for reasoning in $\logicmin$ still remains open. It would be interesting to see 
  if it is actually possible to improve the $\nexpnp($co-$\nexpnp)$ upper bounds. If that were the case, there is a possibility to identify a corresponding fragment from 
  \emph{concept-circumscribed} knowledge bases with a better complexity than 
$\nexpnp($co-$\nexpnp)$.
  
  As a different aspect, it can be seen that the logic $\logic$ and our proposed extension $\logicplus$ impose syntactic 
  restrictions on the use of the typicality operator. First, it is not possible to use a typicality operator under a role operator. Second, 
  only subsumption statements of the form $\T(A) \sqsubseteq C$ are allowed in the TBox. 
  The latter, seems to come from the fact that $\logic$ is based on the approach to propositional 
non-monotonic reasoning proposed in \cite{DBLP:journals/ai/LehmannM92}, where a conditional 
assertion of the form $A \q C$ is used to express that $A$'s \emph{normally} have property 
$C$.
  
  As an example, by lifting these syntactic restrictions, one will be able to express things like:\[\T(\mathsf{Senior\_Teacher}) \sqsubseteq 
\mathsf{Excellent\_Teacher}\] \[\T(\mathsf{Student}) \sqsubseteq \forall \mathsf{attend}.(\mathsf{Class} \sqcap 
\exists\mathsf{imparted}.\T(\mathsf{Senior\_Teacher}))\] 

This allows to relate the typical instances from different classes in a way which is not possible with the current syntax. From a complexity point of view, it is not difficult to observe that the given translations in the paper will also be applicable in this case, without increasing the overall 
    complexity. The reason is that after lifting the mentioned syntactic restrictions, the 
occurrences of $\T_i(A)$ in an extended concept can still be seen as basic concepts.
    
    Therefore, it would be interesting to study what are the effects of removing these 
restrictions, with respect to the kind of conclusions that would be obtained from a 
knowledge base expressed in the resulting non-monotonic logic.
\section{Acknowledgements}
I thank my supervisors Gerhard Brewka and Franz Baader for helpful discussions.
\bibliographystyle{aaai}
\bibliography{external}
%\clearpage
%\input paper_appendix
\end{document}